\documentclass{ecai}

\usepackage{amsfonts}
\usepackage{amsmath}
\usepackage{amsthm}
\usepackage{dsfont}
\usepackage{bbm}
\usepackage{bm}
\usepackage{color}
\usepackage{dirtytalk}
\usepackage{times}
\usepackage{enumerate}
\usepackage{listings}
\usepackage{mathtools}
\usepackage{natbib}
\usepackage{times}
\usepackage{xspace}
\usepackage{algorithm}
\usepackage{algpseudocode}
\usepackage[labelformat=simple]{subcaption}
\usepackage[normalem]{ulem}
\usepackage[utf8]{inputenc}
\usepackage{csquotes}
\usepackage{nicefrac}

\usepackage{soul}
\usepackage{url}
\usepackage[hidelinks]{hyperref}

\usepackage{graphicx}
\usepackage{booktabs}
\usepackage[switch]{lineno}

\usepackage[textsize=tiny]{todonotes}

\usepackage[capitalize,noabbrev]{cleveref}

\usepackage{thmtools}
\declaretheorem[name=Theorem,refname={Theorem,Theorems},Refname={Theorem,Theorems}]{theorem}
\declaretheorem[name=Lemma,refname={Lemma,Lemmas},Refname={Lemma,Lemmas},sibling=theorem]{lemma}

\newcommand{\cD}{\mathcal{D}}
\newcommand{\cE}{\mathcal{E}}

\newcommand{\cG}{\mathcal{G}}

\newcommand{\cL}{\mathcal{L}}

\newcommand{\cT}{\mathcal{T}}

\newcommand{\cX}{\mathcal{X}}

\newcommand{\condprob}[2]{\mathbb{P} \left(#1 \,\middle|\, #2\right)}

\newcommand{\abs}[1]{\left|#1\right|}

\newcommand*\dif{\mathop{}\!\mathrm{d}}

\newcommand{\I}[1]{\mathds{1} \! \left\{#1\right\}}

\newcommand{\set}[1]{\left\{#1\right\}}

\DeclareMathOperator*{\argmax}{arg\,max\,}

\mathchardef\mhyphen="2D

\definecolor{green}{rgb}{0.0, 0.5, 0.0}
\definecolor{black}{rgb}{0.0, 0.0, 0.0}
\renewcommand\paragraph[1]{\vspace{8pt}\noindent\textbf{#1: }}

\begin{document}

%%%%%%%%%%%%%%%%%%%%%%%%%%%%%%%%%%%%%%%%%%%%%%%%%%%%%%%%%%%%%%%%%%%%%%%%

\begin{frontmatter}

%%% Use this command to specify your submission number.
%%% In doubleblind mode, it will be printed on the first page.

\paperid{793} 

%%% Use this command to specify the title of your paper.

\title{Pessimistic Off-Policy Optimization for Learning to Rank}

%%% Use this combinations of commands to specify all authors of your 
%%% paper. Use \fnms{} and \snm{} to indicate everyone's first names 
%%% and surname. This will help the publisher with indexing the 
%%% proceedings. Please use a reasonable approximation in case your 
%%% name does not neatly split into "first names" and "surname".
%%% Specifying your ORCID digital identifier is optional. 
%%% Use the \thanks{} command to indicate one or more corresponding 
%%% authors and their email address(es). If so desired, you can specify
%%% author contributions using the \footnote{} command.

\author[A,B]{\fnms{Matej}~\snm{Cief}}
\author[C]{\fnms{Branislav}~\snm{Kveton}}
\author[B]{\fnms{Michal}~\snm{Kompan}} 

\address[A]{Brno University of Technology}
\address[B]{Kempelen Institute of Intelligent Technologies}
\address[C]{Amazon}

%%% Use this environment to include an abstract of your paper.

\begin{abstract}
 Off-policy learning is a framework for optimizing policies without deploying them, using data collected by another policy. In recommender systems, this is especially challenging due to the imbalance in logged data: some items are recommended and thus logged more frequently than others. This is further perpetuated when recommending a list of items, as the action space is combinatorial. To address this challenge, we study pessimistic off-policy optimization for learning to rank. The key idea is to compute lower confidence bounds on parameters of click models and then return the list with the highest pessimistic estimate of its value. This approach is computationally efficient, and we analyze it. We study its Bayesian and frequentist variants and overcome the limitation of unknown prior by incorporating empirical Bayes. To show the empirical effectiveness of our approach, we compare it to off-policy optimizers that use inverse propensity scores or neglect uncertainty. Our approach outperforms all baselines and is both robust and general.
\end{abstract}

\end{frontmatter}

\section{Introduction}
\label{sec:introduction}

Off-policy optimization can be used to learn better policies when the deployment and testing of sub-optimal solutions is costly, such as in recommender systems \citep{hong_non-stationary_2021}.
Despite the obvious benefits, off-policy optimization is often impeded by the \emph{feedback loop}, where the currently deployed policy influences future training data \citep{jeunen_pessimistic_2021}. This bias in data is one of the main challenges in off-policy optimization.

Several unbiased approaches exist to learn from biased data. \emph{Inverse propensity scoring (IPS)}, which re-weights observations with importance weights \citep{horvitz_generalization_1951} to estimate a policy value, is popular. This so-called off-policy evaluation is often used in off-policy optimization, finding the policy with the highest estimated value \citep{jeunen_pessimistic_2021}. While IPS is popular in practice~\citep{bottou_counterfactual_2013}, it has variance issues that compound at scale, which may prevent a successful deployment \citep{gilotte_offline_2018}. An important scenario where IPS has a high variance is recommending a ranked list of items. In this case, the action space is combinatorial, as the number of ranked lists, which represent actions, is exponential in the length of the lists.

Therefore, in real-world ranking problems (for example, news, web search, and e-commerce), model-based methods often outperform IPS methods \citep{jeunen_joint_2020}. The model-based methods rely on an explicit model of the reward conditioned on a context-action pair, such as the probability of a user clicking on a recommendation \citep{he_practical_2014}. A prevalent approach to fitting model parameters, \emph{maximum likelihood estimation (MLE)}, is impacted by non-uniform data collection. As an example, consider choosing between two restaurants where the first has an average rating of $5$, estimated from five reviews, and the second has an average rating of $4.8$, estimated from a thousand reviews. Optimizers using MLE would choose the first restaurant, as they consider only the average rating, while the second choice is more robust.

In our work, we account for the uncertainty caused by an unevenly explored action space by applying pessimism to reward models of action-context pairs for learning to rank. The challenge is to design lower confidence bounds that hold jointly for all lists as the number of unique lists grows exponentially with the list length. A naïve application of existing pessimistic methods to each unique list is sample inefficient. Also, user behavior signals are often biased for higher-ranked items and can only be collected on items that users actually see. The main contributions of our paper are:

\begin{itemize}
    \item We propose \emph{lower confidence bounds (LCBs)} on parameters of model-based approaches in learning to rank and derive error bounds for acting on them in off-policy optimization.
    \item We study both Bayesian and frequentist approaches to estimating LCBs, including an empirical estimation of the prior, as it is often unknown in practice.
    \item We conduct extensive experiments that show the superiority of the proposed methods compared to IPS and MLE policies on four real-world learning to rank datasets with a large action space.
\end{itemize}

\section{Related Work}
\label{sec:related work}

\paragraph{Off-Policy Optimization}
One popular approach to learning from bandit feedback is to employ the empirical risk minimization principle with IPS estimators \citep{joachims_unbiased_2017, bottou_counterfactual_2013, swaminathan_counterfactual_2017}.
Another popular choice is model-based methods \citep{dudik_doubly_2014}. 
These approaches learn a reward model for specific context-action pairs, which is then used to derive an optimal policy. Due to model misspecification, model-based methods are typically biased but have more favorable variance properties than IPS \citep{jeunen_pessimistic_2021}. Variance issues of IPS estimators are further perpetuated in learning to rank since the action space is exponentially large.

\paragraph{Counterfactual Learning to Rank}
Learning to rank models are often trained from user behavior data \citep{joachims_optimizing_2002}. However, implicit feedback, such as user clicks, is noisy and affected by various biases \citep{joachims_evaluating_2007}. Many studies have explored how to extract unbiased relevance signals from biased clicks. One approach is modeling how the user examines a list using click models \citep{chapelle_expected_2009, chuklin_click_2015}. While IPS estimators for various click models have been studied in the past \citep{li_offline_2018}, the key assumption was that the value of a list is linear in the contributions of individual items in it. IPS estimators are unbiased, and their variance can be controlled in various ways \citep{joachims_deep_2018, swaminathan_counterfactual_2017, swaminathan_self-normalized_2015}, such as clipping of the propensity weights \citep{ionides_truncated_2008}. Nevertheless, they do not model non-linear reward models.

Model-based methods can model non-linearity. Despite this, prior works on counterfactual learning to rank focused mainly on linear estimators \citep{li_offline_2018, swaminathan_off-policy_2017, joachims_unbiased_2017}. More recently, a non-linear doubly-robust method for the cascade model has been proposed~\citep{kiyohara_doubly_2022}. These methods suffer from biases in unexplored parts of the action space. They can be used for optimization but suffer from an overly optimistic estimation, a phenomenon known as the optimiser's curse \citep{smith_optimizers_2006}. Our proposed method works with both linear and non-linear click models while alleviating the optimiser's curse.

\paragraph{Pessimistic Off-Policy Optimization}
While off-policy methods learn from data collected by a different policy, on-policy methods learn from the data that they collect. In online learning, the policy needs to balance the immediate rewards of actions and their information gain \citep{mcinerney_explore_2018}. A common approach is being optimistic with respect to the rewards of actions, and \emph{upper confidence bounds (UCBs)} have proved to be effective \citep{li_contextual-bandit_2010}.

In the offline setting, new data cannot be collected, and a common approach is to act pessimistically to be robust. Prior works~\citep{jeunen_pessimistic_2021,hong_multi-task_2023} derived pessimistic Bayesian LCBs for reward models and showed that this can be lead to major gains. Principled Bayesian methods can be used to obtain closed-form expressions, but this requires knowing the prior and is restricted to specific model classes~\citep{jeunen_pessimistic_2021, chapelle_expected_2009, li_contextual-bandit_2010}.

This is the first work that applies pessimism to learning reward functions for ranking. While pessimism has been popular in offline reinforcement learning \citep{xie_bellman-consistent_2021}, it was applied only to unstructured action spaces before \citep{jeunen_pessimistic_2021}. We extend this work from pointwise to listwise pessimism and study multiple approaches for constructing pessimistic estimates.

\section{Setting}
\label{sec:setting}

We start by introducing our setting. Specifically, we formally define a ranked list, how a user interacts with it, and how the data for off-policy optimization are collected.

We consider the following general model of a user interacting with a ranked list of items. Let $\cE$ be a \emph{ground set of items}, such as all web pages or movies that can be recommended. Let $\Pi_K(\cE)$ be the set of all lists of length $K$ over items $\cE$. A user is recommended a ranked list of items. We denote a \emph{ranked list} with $K$ items by $A = (a_1, \dots, a_K) \in \Pi_K(\cE)$, where $a_k \in \cE$ is the item at position $k$. The user clicks on items in the list, and we observe click indicators on all positions $Y = (Y_1, \dots, Y_K)$, where $Y_k \in \{0, 1\}$ is the \emph{click indicator} on position $k$. The list is chosen as a function of \emph{context} $X \in \cX$, where $X$ can be a user profile or a search query coming from a set of contexts $\cX$.

A ranking \emph{policy} $\pi(\cdot \mid X)$ is a conditional probability distribution over lists given context $X$. It interacts with users for $n$ rounds indexed by $t \in [n]$. In round $t$, it selects a list $A_t \sim \pi(\cdot \mid X_t)$, where $A_t = (a_{t, 1}, \dots, a_{t, K}) \in \Pi_K(\cE)$ and $X_t$ is the context in round $t$. After that, it observes clicks $Y_t = (Y_{t, 1}, \dots, Y_{t, K})$ on all recommended items in the list. All interactions are recorded in a \emph{logged dataset} $\cD = \{(X_t, A_t, Y_t)\}_{t = 1}^n$. The policy that collects $\cD$ is called the \emph{logging policy}, and we denote it by $\pi_0$.

Our goal is to find a policy that recommends the \emph{optimal list} in every context. The optimal list in context $X$ is defined as
\begin{align}
  \textstyle
  A_{*, X}
  = \argmax_{A \in \Pi_K(\cE)} V(A, X)\,,
  \label{eq:optimal list}
\end{align}
where $V(A, X)$ is the value of list $A$ in context $X$. This can be the expected number of clicks or the probability of observing a click. The definition of $V$ depends on the chosen user interaction model. We present several of them in \cref{sec:structured pessimism}.

\section{Pessimistic Optimization}
\label{sec:pessimistic optimization}

Suppose that we want to find list $A_{*, X}$ in \eqref{eq:optimal list} but $V(A, X)$ is unknown. Then, the most straightforward approach is to estimate it and choose the best list according to the estimate. As an example, let $\hat{V}(A, X)$ be the \emph{maximum likelihood estimate (MLE)} of $V(A, X)$. Then the best empirically-estimated list in context $X$ is
\begin{align}
  \textstyle
  \hat{A}_X
  = \argmax_{A \in \Pi_K(\cE)} \hat{V}(A, X)\,.
  \label{eq:mle list}
\end{align}
This solution is problematic when $\hat{V}$ estimates $V$ poorly. The reason is that we may choose a list with a high estimated value $\hat{V}(\hat{A}_X, X)$ but low actual value $V(\hat{A}_X, X)$.

To account for uncertainty, prior works in bandits and reinforcement learning designed pessimistic \emph{lower confidence bounds (LCB)} and acted on them \citep{jin_is_2021}. We adopt the same design principle in our proposed method (\cref{alg:off-policy optimization}). At a high level, the algorithm first computes LCBs $L(X, A)$ for all action-context pairs $(A, X)$. The lower confidence bound satisfies $L(A, X) \leq V(A, X)$ with a high probability. Then it takes an action $\hat{A}_X$ with the highest lower confidence bound $L(\cdot, X)$ in each context $X$. In \cref{sec:structured pessimism,sec:lcbs}, we show how to design LCBs for entire lists of items efficiently. These LCBs, and our subsequent analysis in \cref{sec:analysis}, are our main technical contributions.

\begin{algorithm}[t]
  \caption{Conservative off-policy optimization.}
  \label{alg:off-policy optimization}
  \begin{algorithmic}
    \State \textbf{Inputs:} Logged dataset $\cD$
    \For{$X \in \cX$}
      \State $\hat{A}_X \gets \argmax_{A \in \Pi_K(\cE)} L(A, X)$
    \EndFor
    \State \textbf{Output:} $\hat{A} = (\hat{A}_X)_{X \in \cX}$
  \end{algorithmic}
\end{algorithm}

Lower confidence bounds are beneficial when $\hat{V}$ does not approximate $V$ uniformly well. Specifically, suppose that $\hat{V}$ approximates $V$ better around optimal solutions $A_{*, X}$. This is common in practice because deployed logging policies $\pi_0$ are already optimized to select high-value items. Then, low-value items can only be chosen if the LCBs of high-value items are low. This cannot happen because the high-value items are logged frequently; and thus their estimated mean values are high and their confidence intervals are tight.

As a concrete example, consider two lists of recommended items. The first list contains items with an estimated click-through rate (CTR) of $1$, but all of them were recommended only once. The other list contains items with an estimated CTR of $0.5$, but those items are popular and were recommended a thousand times. Off-policy optimization with the MLE estimator would choose the first list, whose estimated value is high, but the actual value may be low. Off-policy optimization with LCBs would choose the other list since its estimated value is reasonably high but more certain.

\section{Structured Pessimism}
\label{sec:structured pessimism}

In this section, we construct lower confidence bounds for lists. The main challenge is how to establish useful LCBs for all lists jointly since there can be exponentially many lists. To do that, we rely on user-interaction models with ranked lists, the so-called click models~\citep{chuklin_click_2015}. The models allow us to construct LCBs for the whole list by decomposing them into LCBs of items in it.

To illustrate the generality of our approach, we study three popular click models. To simplify notation, we assume that the context $X$ is fixed in this section and drop it from all terms. In each click model, the relevance of item $a \in \cE$ is represented by its attraction probability $\theta_a \in [0, 1]$. This is the probability that the item is clicked, given that it is examined. In each model, we show that when we have  LCBs for each $\theta_a$, we have  LCBs for all lists $A$, trivially by the union bound.

\subsection{Cascade Model}
\label{sec:cascade model}

The \emph{cascade model (CM)} \citep{richardson_predicting_2007,craswell_experimental_2008} assumes that a user examines items in a list from top to bottom until they find a relevant item and click on it \citep{chuklin_click_2015}. Item $a_k$ at position $k$ is examined if and only if item $a_{k-1}$ is examined but not clicked. The first position is always examined.

Since at most one item is clicked, a natural choice for the value of list $A$ is the probability of a click defined as 
\begin{align}
  V_\textsc{cm}(A)
  = 1 - \prod_{k=1}^K (1 - \theta_{a_k})\,,
  \label{eq:value cascade model}
\end{align}
where $\theta_a \in [0, 1]$ denotes the attraction probability of item $a \in \cE$. To stress that the above value is for a specific model, the CM, in this case, we write $V_{\textsc{cm}}$. The optimal list $A_*$ contains $K$ items with the highest attraction probabilities \citep{kveton_cascading_2015}.

To establish LCBs for all lists, we need LCBs for all model parameters. In the CM, the value of a list depends only on the attraction probabilities of its items. Let $L(a)$ be the LCB on the attraction probability of item $a$, where $\theta_a \geq L(a)$ holds with probability at least $1 - \delta$. Then, for all lists $A$ jointly, the LCB
\begin{align*}
  L_\textsc{cm}(A)
  = 1 - \prod_{k = 1}^K (1 - L(a_k))
  \leq 1 - \prod_{k = 1}^K (1 - \theta_{a_k})
\end{align*}
holds with probability at least $1 - \delta \abs{\cE}$, by the union bound over all items. The above inequality holds because we have a lower bound on each term in the product.

\subsection{Dependent-Click Model}
\label{sec:dependent-click model}

The \emph{dependent-click model (DCM)} \citep{guo_efficient_2009} extends the CM to multiple clicks. This model assumes that after a user clicks on an item, they may continue examining items at lower positions in the list. Specifically, the probability that the user continues to explore after a click at position $k \in [K]$ is $\lambda_k \in [0, 1]$.

A natural choice for the value of list $A$ in the DCM is the probability of a satisfactory click, a click upon which the user leaves satisfied. This can be formally written as
\begin{align}
  V_\textsc{dcm}(A)
  = 1 - \prod_{k = 1}^K (1 - (1 - \lambda_k) \theta_{a_k})\,,
  \label{eq:value dependent-click model}
\end{align}
where $\theta_a \in [0, 1]$ denotes the attraction probability of item $a \in \cE$, identically to \cref{sec:cascade model}. The optimal list $A_*$ contains $K$ items with the highest attraction probabilities, where the $k$-th most attractive item is placed at the $k$-th most satisfactory position \citep{katariya_dcm_2016}. Let $L(a)$ be defined as in \cref{sec:cascade model}. Then, for all lists $A$ jointly, the LCB
\begin{align*}
  L_\textsc{dcm}(A)
  & = 1 - \prod_{k = 1}^K (1 - (1 - \lambda_k) L(a_k)) \\
  & \leq 1 - \prod_{k = 1}^K (1 - (1 - \lambda_k) \theta_{a_k})
\end{align*}
holds with probability at least $1 - \delta \abs{\cE}$, by the union bound over all items. To simplify exposition, we assume that the position parameters $\lambda_k$ are known.

\subsection{Position-Based Model}
\label{sec:position-based model}

The \emph{position-based model (PBM)} \citep{craswell_experimental_2008} assumes that the click probability depends on the item and its position. The effect of the position is modeled through the examination probability $p_k \in [0, 1]$ of position $k \in [K]$. Specifically, the item is clicked only if its position is examined and the item is attractive.

Since the PBM allows multiple clicks, a natural choice for the value of list $A$ is the expected number of clicks
\begin{align}
  V_\textsc{pbm}(A)
  = \sum_{k = 1}^K p_k \theta_{a_k}\,,
  \label{eq:value position-based model}
\end{align}
where $\theta_a \in [0, 1]$ denotes the attraction probability of item $a \in \cE$, identically to \cref{sec:cascade model}. The optimal list $A_*$ contains $K$ items with the highest attraction probabilities, where the $k$-th most attractive item is placed at the position with the $k$-th highest $p_k$. Let $L(a)$ be defined as in \cref{sec:cascade model}. Then, for all lists $A$ jointly, the LCB
\begin{align*}
  L_\textsc{pbm}(A)
  = \sum_{k = 1}^K p_k L(a_k)
  \leq \sum_{k = 1}^K p_k \theta_{a_k}
\end{align*}
holds with probability at least $1 - \delta \abs{\cE}$, by the union bound over all items. We assume that the position examination probabilities $p_k$ are known, similarly to the DCM (\cref{sec:dependent-click model}).

\section{Lower Confidence Bounds on Attraction Probabilities}
\label{sec:lcbs}

In this section, we construct LCBs for attraction probabilities $\theta_a$ of individual items in \cref{sec:structured pessimism}. Note that they are means of Bernoulli random variables, which we use in our derivations. We consider two types of LCBs: Bayesian and frequentist. The Bayesian bounds assume that the attraction probabilities are drawn i.i.d.\ from a known prior distribution. The frequentist bounds make no assumption on the distribution of the attraction probabilities. The Bayesian bounds are more practical when the prior is available, while the frequentist bounds are more robust due to fewer modeling assumptions. Our bounds are derived independently for each context $X$. To simplify notation, we drop it in the derivations in this section.

\subsection{Bayesian Lower Confidence Bounds}
\label{sec:bayesian lcbs}

Let $\theta_a \in [0, 1]$ be the mean of a Bernoulli random variable representing the attraction probability of item $a$. Let $Y_1, \dots, Y_n$ be $n$ i.i.d. observations of $\theta_a$. Let the number of positive and negative observations be $n_{a, +}$ and $n_{a, -}$, respectively. We estimate $n_{a, +}$ and $n_{a, -}$ for each click model in \cref{sec:click models}. In the Bayesian setting, we make an assumption that $\theta_a \sim \mathrm{Beta}(\alpha, \beta)$. Thus $\theta_a \mid Y_1, \dots, Y_n \sim \mathrm{Beta}(\alpha + n_{a, +}, \beta + n_{a, -})$ \citep{bishop_pattern_2006} and a lower confidence bound on $\theta_a$ that holds with probability at least $1 - \delta$ is $L(a) = $
\begin{align}
  \max \set{\ell \in [0, 1]: \smashoperator{\int_{z = 0}^\ell} \mathrm{Beta}(z; \alpha + n_{a, +}, \beta + n_{a, -}) \dif z \leq \frac{\delta}{2}}\,.
  \label{eq:bayes quantile}
\end{align}
According to \eqref{eq:bayes quantile}, $L(a)$ is the largest value such that the probability of $\theta_a \leq L(a)$ is at most $\frac{\delta}{2}$ \citep{bishop_pattern_2006}. We note that $L(a)$ is a quantile of a probability density.

\subsection{Frequentist Lower Confidence Bounds}
\label{sec:frequentist lcbs}

If the prior of $\theta_a$ is unknown or poorly estimated, Bayesian estimates could be biased. In this case, Hoeffding's inequality \citep{vershynin_high-dimensional_2018} would be preferred, as it provides a confidence interval for any random variable on $[0, 1]$. Specifically, let $\theta_a$ be any value in $[0, 1]$, and all other quantities be defined as in the Bayesian estimator. Then a $1 - \delta$ lower confidence bound on $\theta_a$ is
\begin{align}
  L(a)
  = n_{a, +} / n_a - \sqrt{\log(1 / \delta) / (2n_a)}\,,
  \label{eq:hoeffding inequality}
\end{align}
where $n_{a, +} / n_a$ is the MLE, $n_a = n_{a, +} + n_{a, -}$, and event $\theta_a \leq L(a)$ occurs with probability at most $\delta$ \citep{vershynin_high-dimensional_2018}.

\subsection{Prior Estimation}
\label{sec: prior estimation}

One shortcoming of Bayesian methods is that the prior is often unknown. To address this, we show how to estimate it using empirical Bayes \citep{maritz_empirical_2017}, which can be implemented for attraction probabilities as follows. For any $a \in \cE$, let $\theta_a \sim \mathrm{Beta}(\alpha, \beta)$ be the mean of a Bernoulli random variable, which is drawn i.i.d.\ from the unknown prior $\mathrm{Beta}(\alpha, \beta)$. Let $N_{a, +}$ and $N_{a, -}$ be the random variables that denote the number of positive and negative observations, respectively, of $\theta_a$. Let $n_{a, +}$ and $n_{a, -}$ be their observed values, and $n_a = n_{a, +} + n_{a, -}$. Then, the likelihood of the observations for any fixed $\alpha$ and $\beta$ is
\begin{align}
  & \cL(\alpha, \beta)
  = \prod_{a \in \cE} \condprob{N_{a, +} = n_{a, +}, \, N_{a, +} = n_{a, +}}{\alpha, \beta}
  \nonumber \\
  & = \prod_{a \in \cE} \frac{\Gamma(\alpha + \beta)}{\Gamma(\alpha) \Gamma(\beta)}
  \int_{z = 0}^1 z^{\alpha + n_{a, +} - 1}
  (1 - z)^{\beta + n_{a, -} - 1} \dif z
  \label{eq: prior estimation} \\
  & = \prod_{a \in \cE}
  \frac{\Gamma(\alpha + \beta) \Gamma(\alpha + n_{a, +}) \Gamma(\beta + n_{a, -})}
  {\Gamma(\alpha) \Gamma(\beta) \Gamma(\alpha + \beta + n_a)}\,.
  \nonumber
\end{align}
The last equality follows from the fact that
\begin{align*}
  \int_{z = 0}^1
  \frac{\Gamma(\alpha + \beta + n_a) z^{\alpha + n_{a, +} - 1} (1 - z)^{\beta + n_{a, -} - 1}}{\Gamma(\alpha + n_{a, +}) \Gamma(\beta + n_{a, -})} \dif z
  = 1\,.
\end{align*}
Empirical Bayes \citep{maritz_empirical_2017} is a statistical procedure for finding $(\alpha, \beta)$ that maximize $\cL(\alpha, \beta)$. To find them, we search on a grid. Specifically, let $\cG = [m]$ for some integer $m > 0$. Then we search over all $(\alpha, \beta) \in \cG^2$. The grid search is feasible since the parameter space has only $2$ dimensions.

\section{Analysis}
\label{sec:analysis}

The analysis is structured as follows. First, we derive confidence intervals for items and lists. Then, we show how the error of acting with respect to LCBs is bounded. Finally, we discuss how different choices of $\pi_0$ affect the error in \cref{lemma:optimization lcb}. All proofs are presented in \cref{appendix: pessimistic optimization}. We conduct a frequentist analysis based on the confidence intervals in \cref{sec:frequentist lcbs}. A similar analysis could be done for the Bayesian setting (\cref{sec:bayesian lcbs}). The analysis is exact for the CM and DCM. For the PBM, it is approximate, and we provide details in \cref{sec:click models}.

Let $\theta_{a, X} \in [0, 1]$ be the attraction probability of item $a \in \cE$ in context $X \in \cX$. Let $\hat{\theta}_{a, X} = n_{a, X, +} / n_{a, X}$ be its empirical estimate in \eqref{eq:hoeffding inequality}, where $n_{a, X, +}$ is the number of clicks on item $a$ in context $X$ and $n_{a, X}$ is the number of times that item $a$ is observed in context $X$. Then, we get the following concentration bound for all items.

\begin{lemma}[Concentration of item estimates]
\label{lemma: concentration item} Let
\begin{align*}
  c(a, X)
  = \sqrt{\log(1 / \delta ) / (2 n_{a, X})}\,.
\end{align*}
Then for any item $a \in \cE$ and context $X \in \cX$, $\abs{\hat{\theta}_{a, X} - \theta_{a, X}} \leq c(a, X)$ holds with probability at least $1 - \delta$.
\end{lemma}

Let $V(A, X)$ be the value of list $A \in \Pi_K(\cE)$ in context $X \in \cX$, using the unknown attraction probabilities $\theta_{a, X}$. Let $\hat{V}(A, X)$ be its estimate using $\hat{\theta}_{a, X}$, for any click model in \cref{sec:structured pessimism}. Then, we get the following concentration bound for all lists.

\begin{lemma}[Concentration of list estimates]
\label{lemma: concentration list} Let
\begin{align*}
  c(A, X)
  = \sum_{a \in A}
  \sqrt{\frac{\log(\abs{\cE} \abs{\cX} / \delta)}{2 n_{a, X}}}\,.
\end{align*}
Then for any list $A \in \Pi_K(\cE)$ and context $X \in \cX$, and any click model in \cref{sec:structured pessimism}, $\abs{\hat{V}(A, X) - V(A, X)} \leq c(A, X)$ holds with probability at least $1 - \delta$, jointly over all $A$ and $X$.
\end{lemma}

Now we show how the error due to acting pessimistically does not depend on the uncertainty of the chosen list but on the confidence interval width of optimal list $c(A_{*,X}, X)$. This is desirable when the logging policy already performs well (\cref{sec:pessimistic optimization}).

\begin{theorem}[Error of acting pessimistically]
\label{lemma:optimization lcb} Let $A_{*, X}$ and $\hat{A}_X$ be defined as in \eqref{eq:optimal list} and \cref{alg:off-policy optimization}, respectively. Let $L(A, X) = \hat{V}(A, X) - c(A, X)$ be a high-probability lower bound on the value of list $A$ in context $X$, where all quantities are defined as in \cref{lemma: concentration list}. Then, for any context $X$, the error of acting with respect to a lower confidence bound satisfies
\begin{align*}
  V(A_{*, X}, X) - V(\hat{A}_X, X)
  & \leq 2 c(A_{*, X}, X) \\ 
  & \leq 2 \sum_{a \in A_{*, X}}
  \sqrt{\frac{\log(\abs{\cE} \abs{\cX} / \delta)}{2 n_{a, X}}}
\end{align*}
with probability at least $1 - \delta$, jointly over all $X$.
\end{theorem}

\cref{lemma:optimization lcb} shows that our error depends on the number of observations of items in the optimal list $a \in A_{*,X}$. To illustrate how it depends on the logging policy $\pi_0$, fix context $X \in \cX$ and let $n_X$ be the number of logged lists in context $X$. We discuss two cases.

Suppose that $\pi_0$ is uniform, and thus each item $a \in \cE$ is placed at the first position with probability $1 / \abs{\cE}$. Moreover, suppose that the first position is examined with a high probability. This happens with probability $1$ in the CM (\cref{sec:cascade model}) and DCM (\cref{sec:dependent-click model}), and with a high probability in the PBM (\cref{sec:position-based model}) when $p_1$ is high. Then $n_{a, X} = \tilde{\Omega}(n_X / \abs{\cE})$ as $n_X \to \infty$ and the error bound in \cref{lemma:optimization lcb} becomes $\tilde{O}(K \sqrt{\abs{\cE} / n_X})$, where $\tilde{O}$ and $\tilde{\Omega}$ are big O notations up to logarithmic factors. The bound is independent of the number of lists $\abs{\Pi_K(\cE)}$, which is exponentially large.

Now, suppose that the logging policy is near optimal. One way of formalizing it is that each item $a \in A_{*, X}$ is placed at the first position with probability $1 / K$. Then, under the same assumptions as in the earlier discussion, $n_{a, X} = \tilde{\Omega}(n_X / K)$ as $n_X \to \infty$ and the error bound in \cref{lemma:optimization lcb} becomes $\tilde{O}(K \sqrt{K / n_X})$. This improves by a factor of $\abs{\cE} / K$ upon the earlier discussed bound.

\section{Experiments}

\begin{figure*}[t!]
    \centering
    \begin{subfigure}[t]{0.8\linewidth}
     \includegraphics[width=\textwidth]{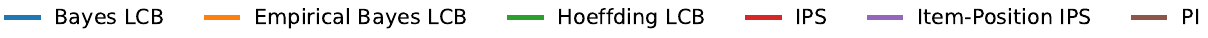}     
    \end{subfigure}
    \begin{subfigure}[t]{\linewidth}
     \includegraphics[width=\textwidth]{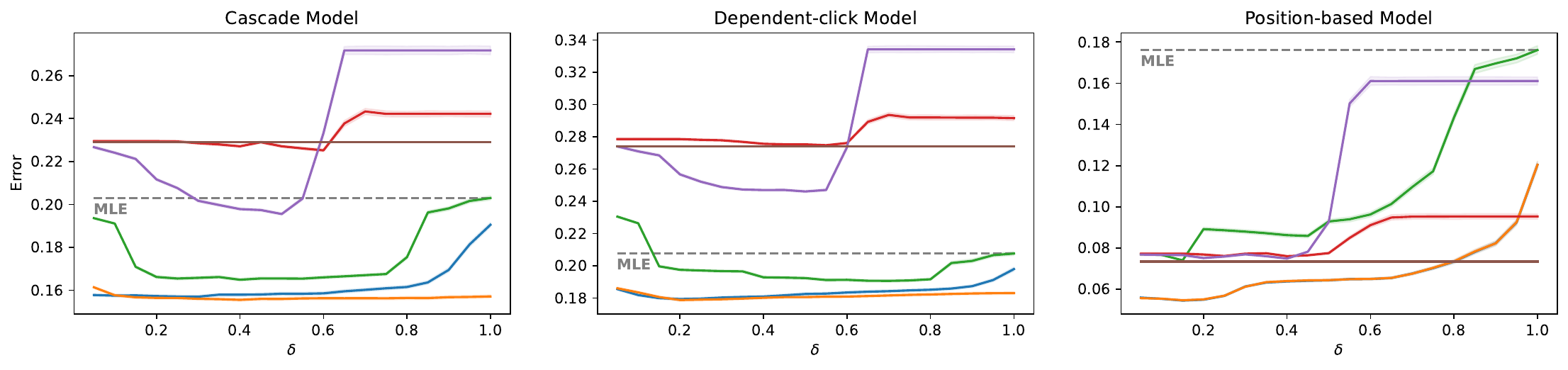}     
    \end{subfigure}
    \caption{Comparison of our methods to baselines on three click models and top 10 queries. We vary the parameter $\delta$ that sets the confidence interval width. MLE is the grey dashed line.}
    \label{fig:02}
\end{figure*}
\begin{figure*}[t!]
\centering
\begin{minipage}[t]{0.8\linewidth}
     \includegraphics[width=\textwidth]{figs/legend.pdf}     
    \end{minipage}
\begin{minipage}[t]{0.34\textwidth}
    \centering
    \includegraphics[width=\textwidth]{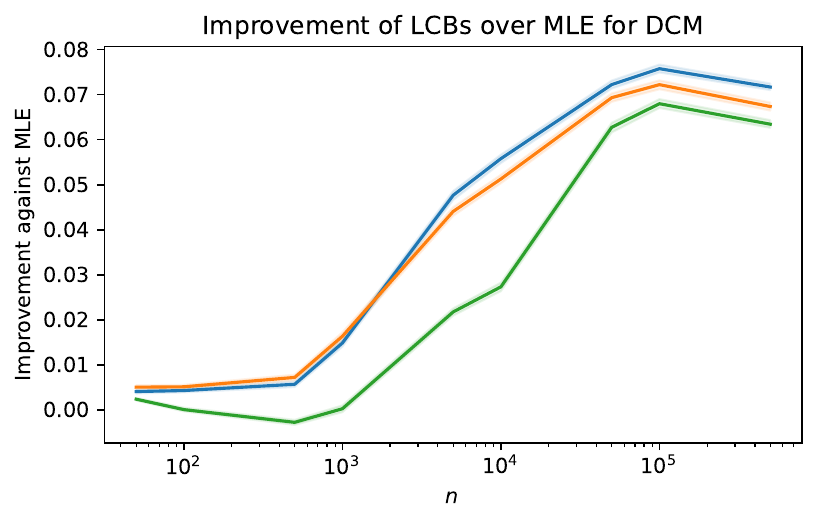}
    \caption{Comparison of our methods to MLE when increasing the sample size $n$.}
    \label{fig:04}
\end{minipage}
\hfill
\begin{minipage}[t]{0.64\textwidth}
    \centering
    \includegraphics[width=\textwidth]{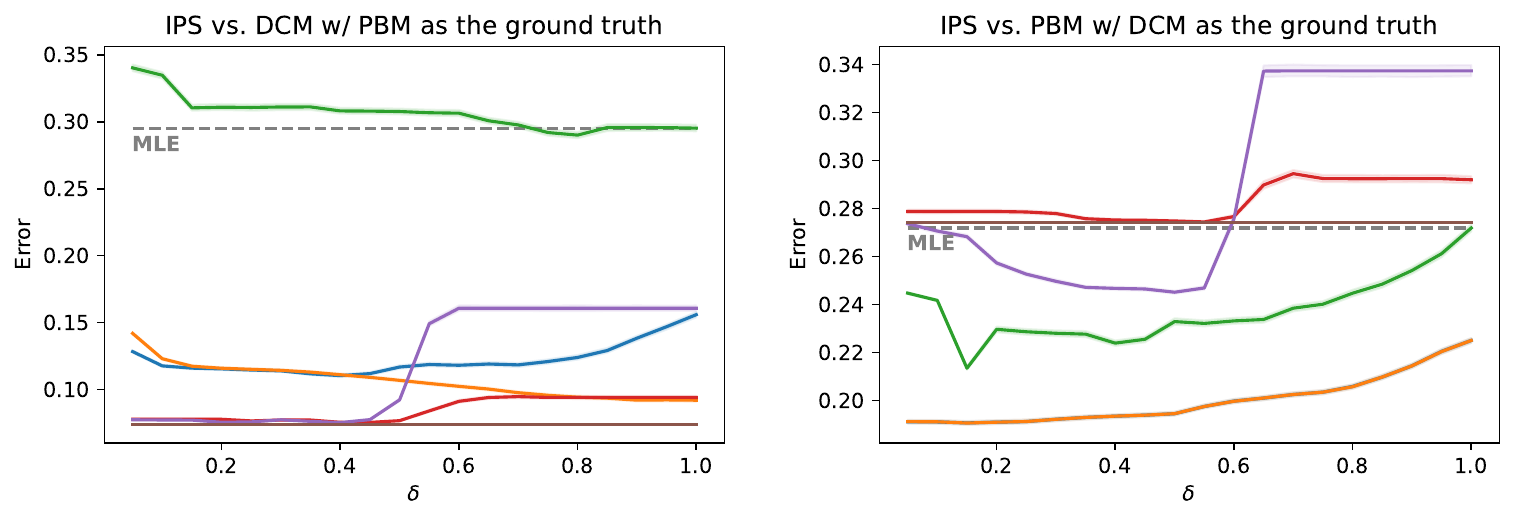}
    \caption{Robustness evaluation of our methods and baselines.}
    \label{fig:07}
\end{minipage}
\vspace{12pt}
\end{figure*}

We conduct extensive experiments where we compare our policies to four baselines: MLE, IPS \citep{robins_estimation_1994}, structured item-position IPS \citep{li_offline_2018}, and pseudo-inverse estimator \citep{swaminathan_off-policy_2017}. We call our method LCB because it optimizes lower confidence bounds.

\subsection{Experimental Setup}
\label{sec: experimental setup}

We use the \emph{Yandex} dataset \citep{yandex_yandex_2013} for the first three experiments. We treat each query as a different context $X$, perform the computations separately, and then average the results. Due to a huge position bias in the dataset, most clicks occur at the first positions; we only keep the first $4$ positions in each list and discard the rest, as done in prior works \citep{li_offline_2018}. We observe improvements without this preprocessing step but they are less pronounced.

All methods are evaluated as follows. We first fit a click model (\cref{sec:structured pessimism}) to a dataset. Because of that, we know the optimal list $A_{*, X}$ in each context $X$ under that model. We select a specific set of queries from the dataset for each experiment. Then, for each query, we select uniformly at random items from that query and generate clicks based on the fitted click model. We repeat this $n$ times and obtain a logged dataset $\cD = \set{(X_t, A_t, Y_t)}_{t = 1}^{n}$, where $X_t$ and $A_t$ are from the original dataset, and $Y_t$ is generated by the click model. Unless stated, $n = 1000$. We apply off-policy methods to $\cD$ to find the most valuable lists $\hat{A}_X$. We evaluate these lists against the true optimal lists $A_{*, X}$ using \emph{error} $\mathbb{E}_X\left[V(A_{*,X}, X) - V(\hat{A}_X, X)\right]$. We estimate the logging policy $\pi_0$ from $\cD$. We repeat each experiment $500$ times and report all mean values with standard errors (shaded areas around lines in the plots).

We experiment with both Bayesian and frequentist lower confidence bounds in \cref{sec:lcbs}. They hold with probability $1 - \delta$, where $\delta \in [0.05, 1]$ is a tunable parameter representing the width of the confidence interval. The estimation of our model parameters is detailed in \cref{sec:click models}. Since Bayesian methods depend on the prior, we also evaluate empirical Bayes for learning the prior (\cref{sec: prior estimation}) with grid $\cG = \{2^{i-1}\}_{i=1}^{10}$.

\subsection{Baselines}
\label{sec:implementing optimizers}

The first baseline is the best list under the same click model as in our method but with MLE-estimated parameters. We also use the following baselines from prior works \citep{ionides_truncated_2008,li_offline_2018,swaminathan_off-policy_2017}.

\paragraph{IPS} We implement an estimator using propensity weights, where the whole list is a unique action. We compute the propensity weights separately for each query. We also add a \emph{tunable clipping parameter $M$} \citep{ionides_truncated_2008}. IPS optimizer then selects $\hat{A}_X$ that maximizes
\begin{align}
  \hat{V}_\textsc{IPS}(A,X)
  = \sum_{t \in \cT_X}
  \min\left\{M, \frac{\I{A_t = A}}{p_{A,X}}\right\}Y_t\,,
\end{align}
where we estimate propensities $p_{A,X} = \frac{\sum_{t \in \cT_X} \I{A_t = A}}{\abs{\cT_X}}$ as the frequency of recommending list $A$, $Y_t$ is the number of clicks in list $A_t$, and $\cT_X$ is the set of all indices $t \in [n]$ such that $X_t = X$. The maximization of $\hat{V}_\textsc{IPS}(A, X)$ is a linear program, where we search over all $A \in \Pi_K(\cE)$ \cite{strehl_learning_2010}.

\paragraph{Item-Position IPS} Similarly to the IPS estimator, we implement a structured IPS estimator using linearity of the item-position model \citep{li_offline_2018}, where the expected value of a list is the sum of attraction probabilities of its item-position entries. The list value is
\begin{align}
  \hat{V}_\textsc{IP-IPS}(A, X)
  = \sum_{t \in \cT_X}\sum_{k=1}^K
  \min\left\{M, \frac{\I{a_{t,k} = a}}{p_{a,k,X}}\right\}Y_{t,k}\,,
\end{align}
where $p_{a,k,X} = \frac{\sum_{t \in \cT_X} \I{a_{t,k} = a}}{\abs{\cT_X}}$ and $\cT_X$ is the set of all $t \in [n]$ such that $X_t = X$. To maximize $\hat{V}_\textsc{IP-IPS}(A, X)$, we select an item with the highest attraction probability for each position $k \in [K]$.

\paragraph{Pseudo-Inverse Estimator (PI)} We also compare to the pseudo-inverse estimator \citep{swaminathan_off-policy_2017}, which assumes that the value of a list is the sum of individual items in it. The context-specific weight vector $\phi_{X}$ can then be learned in a closed form as 
\begin{align}
\hat{\phi}_{X}=\left(\mathbb{E}_{\pi_0}\left[\mathbf{1}_A \mathbf{1}_A^{T} \mid X\right]\right)^{\dagger} \mathbb{E}_{\pi_0}\left[Y \mathbf{1}_A \mid X\right],
\end{align}
where $\mathbf{1}_A \in \{0, 1\}^{K \abs{\cE}}$ is a \emph{list indicator vector} whose entries indicate which item is at which position and $\mathbb{E}_{\pi_0}$ is an expectaion over $A \sim \pi_0(\cdot \mid X)$. We denote by $Y$ the sum of logged clicks on the list $A$ and by $M^\dagger$ the Moore-Penrose pseudo-inverse of a matrix $M$. As mentioned in \citep{swaminathan_off-policy_2017}, the trained regression model can be used for off-policy optimization. We greedily add the most attractive items to the list from the highest position to the lowest.

\subsection{Yandex Results}
\label{sec: results}

The experiments are organized as follows. First, we compare LCBs to all baselines on frequent queries while we vary the confidence interval width represented by parameter $\delta$. Second, we compare LCBs to MLE while automatically learning parameter $\delta$ from data. Finally, we study the robustness of model-based LCB estimators to model misspecification. We refer readers to \cref{appendix: other experiments} for experiments on less frequent queries, where we show that LCBs work well even with less data. We also experiment with the hyperparameters of empirical Bayes from \cref{sec: prior estimation}. Most plots are reported as a function of $\delta$ because it is a tunable parameter of our method. We map the clipping parameter $M$ to $\delta$ as follows.

\begin{table}[h]
\label{table: clipping}
\begin{center}
\begin{small}
\begin{sc}
\resizebox{\columnwidth}{!}{%
\begin{tabular}{lcccccccc}
\toprule
$\delta$ & .05 & .1 & .15 & .2 & .25 & .35 & .45 & .5 \\
$M$ & 1 & 5 & 10 & 50 & 100 & 300 & 500 & 600 \\ 
\midrule
$\delta$ & .55 & .65 & .75 & .8 & .85 & .9 & .95 & 1\\
$M$ & 700 & 900 & 1100 & 1200 & 1300 & 1400 & 1500 & $\infty$\\
\bottomrule
\end{tabular}%
}
\end{sc}
\end{small}
\end{center}
\vskip -18pt
\end{table}

\paragraph{Top 10 Queries}
We start with evaluating all estimators on $10$ most frequent queries in the Yandex dataset. Results in \cref{fig:02} show improvements when using LCBs for all models. Specifically, for almost any $\delta$ in all models, the error is lower than using MLE. When optimizing non-linear list values, such as those in CM and DCM, we outperform all the baselines that assume linearity. In PBM, where the list value is linear, the baselines can perform similarly to the LCB estimators. We observe that the empirical estimation of the prior improves upon an uninformative $\mathrm{Beta}(1, 1)$ prior.

\paragraph{More Realistic Comparison to MLE}
MLE is common in practice and does not have a hyper-parameter $\delta$ to tune, unlike our method. To show that our approach can beat the MLE in a realistic setting, we estimate $\delta$ on past data and then apply it to future data. 
We apply the evaluation protocol from the \emph{Top 10 Queries} experiment on the first $5$ days of data with fixed sample size $n = 1000$ for each query. We select $\delta$ that minimizes the Bayesian LCB error, which is $\delta = 0.2$.
We fix $\delta = 0.2$ and apply the evaluation protocol from the \emph{Top 10 Queries} experiment on the last $23$ days. We report the difference between MLE and Bayesian LCB errors.
This is repeated $500$ times while varying logged sample size $n \in [50, 500\ 000]$.
\cref{fig:04} shows that the largest improvements are at the sample size $100\ 000$. This implies that LCBs have a \emph{sweet spot} where they work the best. We observe smaller improvements for smaller sample sizes because the uncertainty is too high to allow effective modeling. On the other hand, when the sample size is large, the uncertainty is low everywhere and it is not needed to model it.

\paragraph{Robustness to Model Misspecification}
Now, we examine how the estimators behave when the model is misspecified. In the \emph{Top 10 Queries} experiment, we observe that the baselines with linearity assumptions do not perform well in non-linear models, such as CM or DCM, but they perform well when the value of a list is the sum of clicks, such as in PBM. We fit PBM and use it to generate a logged dataset. Then, we use DCM to learn the attraction probabilities for MLE and LCB methods. We also examine the opposite scenario, using DCM as a ground truth model and estimating attraction probabilities with PBM. This does not impact IPS and PI baselines. Our results are reported in \cref{fig:07}. 
In the left plot, we use a non-linear model to fit linear rewards. As a result, MLE and LCB methods are misspecified. Other baselines that assume linearity perform better in this setup. Nevertheless, Bayesian LCBs still have a 50\% lower error than MLE.
In the right plot, the reward is non-linear, and all methods (except IPS) assume linearity in item-level rewards. Here, MLE is comparable to other baselines, and LCBs consistently outperform all baselines. In summary, we show that LCBs are relatively robust to model misspecification, and definitely much more than MLE.

\begin{figure*}[t]
    \centering
    \begin{subfigure}[t]{0.8\linewidth}
     \includegraphics[width=\textwidth]{figs/legend.pdf}     
    \end{subfigure}
    \begin{subfigure}[t]{\linewidth}
        \includegraphics[width=\textwidth]{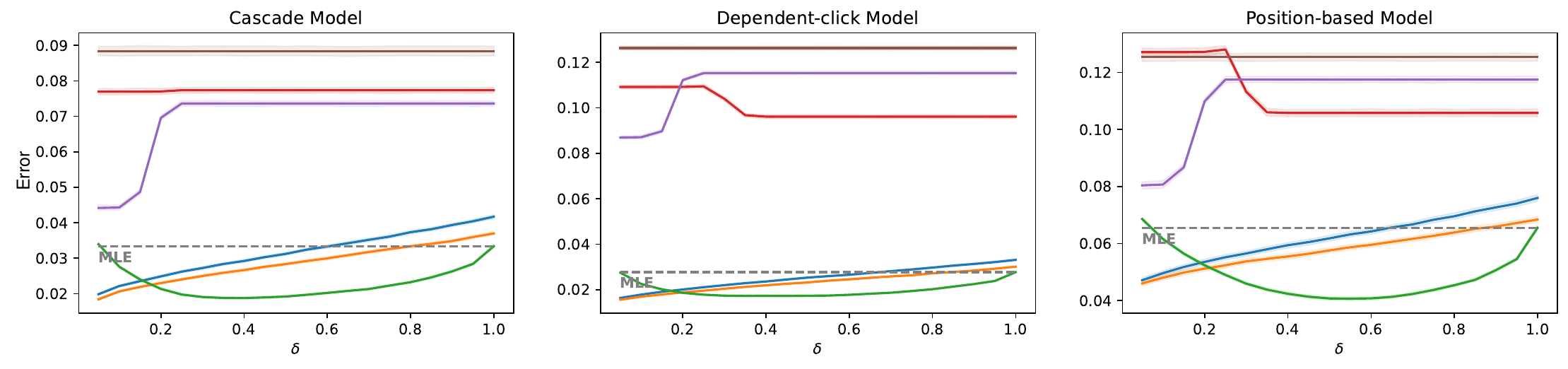}
    \end{subfigure}
    \caption{Comparison of our methods to baselines on the Yahoo! Webscope dataset.}
    \label{fig:01-yahoo}
    \vspace{2pt}
\end{figure*}
\begin{figure*}[t]
    \centering
    \begin{subfigure}[t]{0.8\linewidth}
     \includegraphics[width=\textwidth]{figs/legend.pdf}     
    \end{subfigure}
    \begin{subfigure}[t]{\linewidth}
    \includegraphics[width=\textwidth]{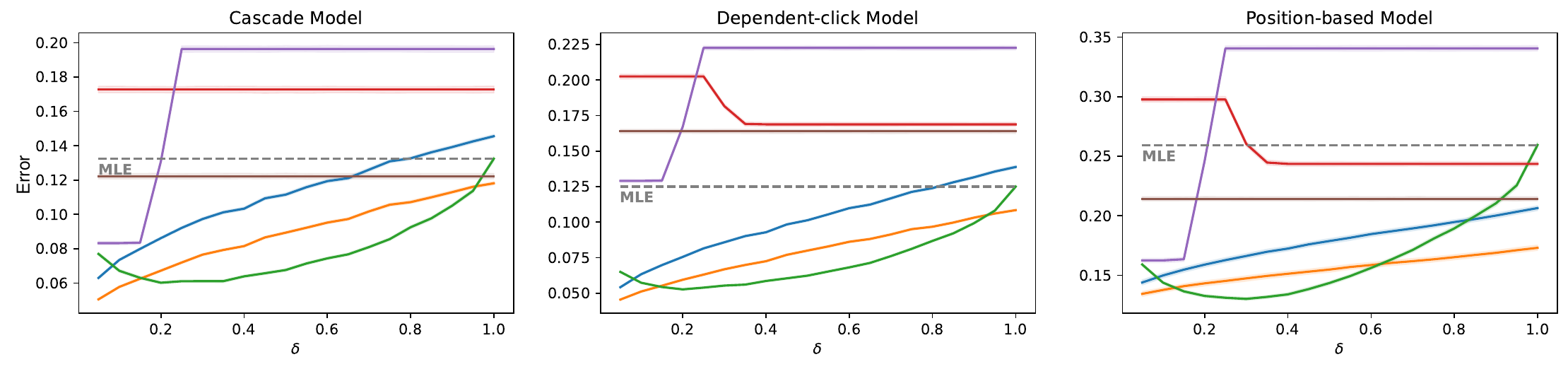}
    \end{subfigure}
    \caption{Comparison of our methods to baselines on the MSLR-WEB30k dataset.}
    \label{fig:01-mslr}
    \vspace{2pt}
\end{figure*}
\begin{figure*}[!t]
    \centering
    \begin{subfigure}[t]{0.8\linewidth}
     \includegraphics[width=\textwidth]{figs/legend.pdf}     
    \end{subfigure}
    \begin{subfigure}[t]{\linewidth}
    \includegraphics[width=\textwidth]{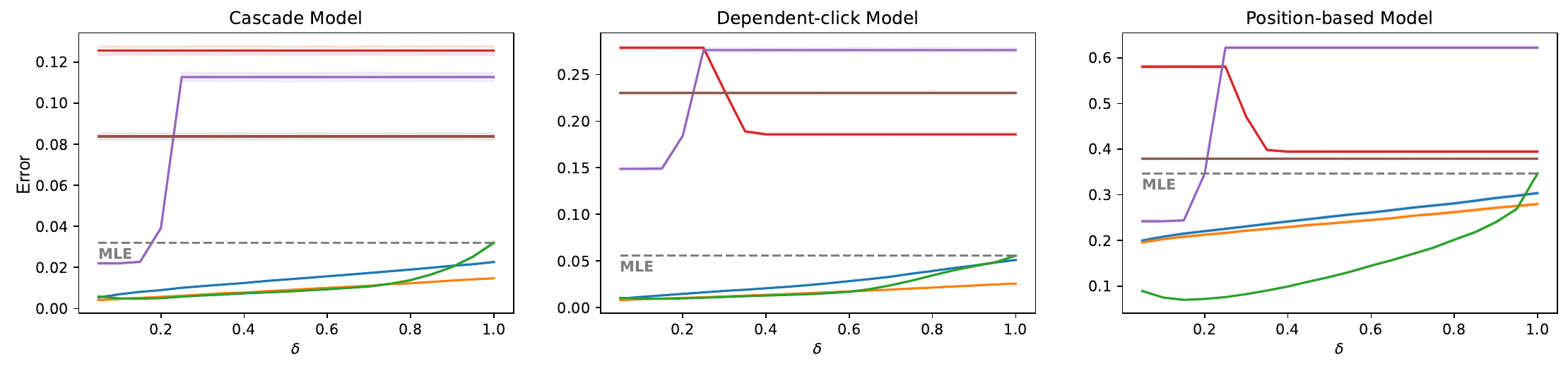}
    \end{subfigure}
    \caption{Comparison of our methods to baselines on the Istella dataset.}
    \label{fig:01-istella}
    \vspace{12pt}
\end{figure*}

\subsection{Results on Other Datasets}
We validated the results on other popular datasets, namely Yahoo! Webscope \citep{yahoo_c14_2010}, Istella \citep{dato_fast_2017}, and MSLR-WEB \citep{qin_introducing_2013}.
These datasets do not contain clicks, only human-labeled query-document relevance scores, with $\mathrm{score}(a) \in \set{0, 1, 2, 3, 4}$ for item $a$ ranked from 0 (not relevant) to 4 (highly relevant). We follow a standard procedure to generate clicks by mapping relevance scores to attraction probabilities based on the \emph{navigational} user model \citep[Table 2,][]{hofmann_fidelity_2013}.

\begin{table}[h]
\vspace{-5pt}
\begin{center}
\begin{small}
\begin{sc}
\begin{tabular}{lccccc}
\toprule
$\mathrm{score}(a)$ & 0 & 1 & 2 & 3 & 4 \\
$\theta_a$ & 0.05 & 0.1 & 0.2 & 0.4 & 0.8 \\ 
\bottomrule
\end{tabular}
\end{sc}
\end{small}
\end{center}
\vspace{-10pt}
\end{table}

\noindent In the PBM, we define position examination probabilities based on an eye-tracking experiment \citep{joachims_accurately_2017}. In the DCM, at position $k \in [K]$, $\lambda_k = 1 - \exp(-k + 0.5) / 0.5$. We randomly select $5000$ queries $q$. Each query has its own set of labeled documents denoted by $\cE_q$. To get a logged dataset, we sample $100$ lists of length $K = 4$ from labeled documents for each query from a Dirichlet distribution with parameters $\alpha = (\theta_a)_{a \in \cE_q}$. We use the same evaluation protocol as in \cref{sec: experimental setup}. \cref{fig:01-yahoo,fig:01-mslr,fig:01-istella} validate our earlier findings. In summary, LCBs outperform MLE and other baselines for most $\delta$ values and provide major improvements. We observed similar results for other sample sizes and list lengths.

\section{Conclusions}

We study pessimistic off-policy optimization in learning to rank. The key idea is to design lower confidence bounds (LCBs) for values of lists through LCBs of items in them. We prove that the error due to choosing the best list in our model is polynomial in the number of items in the list, as opposed to polynomial in the number of lists, which is exponential in the length of the list. We also apply LCBs to non-linear objectives, such as the CM in \eqref{eq:value cascade model} and DCM in \eqref{eq:value dependent-click model}, based on their linearization. This is the first paper in off-policy learning where this approximation was applied and analyzed. Our approach outperforms MLE and IPS methods experimentally. Furthermore, we show that our LCBs are robust to model misspecification and perform better with almost any confidence interval width. We also show how to estimate the prior with empirical Bayes when it is unknown.

One natural future direction is to extend our work to any click model. This can be generally achieved with an ensemble of models, each trained on a different bootstrapped dataset. This, however, presents computational challenges, and further research is needed to address them. We also do not use theory-suggested confidence intervals in experiments because they are too conservative. Therefore, more data-dependent confidence intervals are needed. Finally, the focus of our work is on a large action space of all lists. We wanted this contribution to stand out and thus focus on tabular contexts. An extension to large context spaces should be possible, though.

\begin{ack}
    This research was partially supported by DisAI, a project funded by the European Union under the Horizon Europe, GA No. 101079164, \url{https://doi.org/10.3030/101079164}.
\end{ack}
\clearpage

\bibliography{references}

%%
%% If your work has an appendix, this is the place to put it.
\clearpage
\onecolumn
\appendix

\section{Learning Click Model Parameters}
\label{sec:click models}

To simplify notation, we only consider a small finite number of contexts, such as the day of the week or user characteristics. For each such context, we estimate the attraction probabilities separately. In this section, we show the calculation of model parameters with respective applications of LCB to the three click models mentioned in \cref{sec:structured pessimism}. We assume that the context is fixed, and computations are done over each context separately; therefore, we define $\cT_X$ to be the set of all indices of $[n]$ such that $X_t = X, \; \forall t \in \cT_X$.

\paragraph{Cascade Model}
% \label{sec:cascade model application}
The cascade model has only one type of parameters, attraction probabilities $\theta_a$, and these can be estimated from clicks, as the number of clicks over the number of examinations \citep{kveton_cascading_2015}. According to the cascade model assumptions, items are examined from the top until the user clicks on some item, and the user does not examine anything further. Therefore to model $\theta_{a,X}$, we collect the number of positive impressions $n_{a,X,+} = \sum_{t \in \cT_X} \sum_{k = 1}^K \I{a_{t, k} = a \land Y_{t, k} = 1}$ (user examined and clicked) and the number of negative impressions $n_{a,X,-} = \sum_{t \in \cT_X} \sum_{k = 1}^K \I{a_{t, k} = a}\I {\sum_{j=1}^{k-1} Y_{t, j} = 0}$ (examined, but did not click) for item $a$ in context $X$ and calculate either Bayesian, frequentist LCBs, and prior according to \eqref{eq:bayes quantile}, \eqref{eq:hoeffding inequality}, and \cref{sec: prior estimation}.

\paragraph{Dependent-Click Model}
When fitting the dependent-click model, we process the logged data according to the following assumption: examining items from the top until the final observed click at the lowest position and disregarding all items below. Observed impressions on each item $a$ in context $X$ are sampled from its unknown Bernoulli distribution $\mathrm{Ber}(\theta_{a,X})$. To model $\theta_{a,X}$, we collect the number of positive impressions $n_{a,X,+} = \sum_{t \in \cT_X} \sum_{k = 1}^K \I{a_{t, k} = a}Y_{t, k}$ (clicks) and the number of negative impressions $n_{a,X,-} = \sum_{t \in \cT_X} \sum_{k = 1}^K \I{a_{t, k} = a}(1-Y_{t, k})$ (examined, but not clicked) for item $a$ and calculate either Bayesian, frequentist LCBs, and prior according to \eqref{eq:bayes quantile}, \eqref{eq:hoeffding inequality}, and \cref{sec: prior estimation}. To model the probability $\lambda_{k,X}$, we collect positive observations that the click is the last $n_{k,X,+} = \sum_{t \in \cT_X} \I{\sum_{j = k}^K Y_{t, j} = 1}Y_{t, k}$ and negative observations that user continues exploring as $n_{k,X,-} = \sum_{t \in \cT_X}\I{\sum_{j = k}^K Y_{t, j} > 1}Y_{t, k}$.

\paragraph{Position-Based Model}
To learn the parameters of the position-based model, we use an EM algorithm. We solve $\min_{\theta, p} \sum_{t = 1}^n \sum_{k = 1}^K (\theta_{a_{t, k}, X_t} p_k - Y_{t, k})^2$ by alternating least squares algorithm \cite{takacs_alternating_2012} to obtain an estimate of $p$. Then the propensity-weighted number of positive impressions is $n_{a,X,+} = \sum_{t \in \cT_X} \sum_{k = 1}^K \I{a_{t, k} = a \land Y_{t, k} = 1} / p_k$ and the number of negative impressions $n_{a,X,-} = \sum_{t \in \cT_X} \sum_{k = 1}^K \I{a_{t, k} = a \land Y_{t, k} = 0} / p_k$.

\section{Proofs of Pessimistic Optimization}
\label{appendix: pessimistic optimization}

\paragraph{Proof of \cref{lemma: concentration list}}\\
\paragraph{PBM} By \cref{lemma: concentration item}, for any item $a \in \cE$ and context $X \in \cX$, we have that
\begin{align*}
  |\theta_{a, X} - \hat{\theta}_{a, X}|
  \leq \sqrt{\log(1 / \delta ) / (2 n_{a, X})}
\end{align*}
holds with probability at least $1 - \delta$. Therefore, by the union bound, we have that
\begin{align}
  |\theta_{a, X} - \hat{\theta}_{a, X}|
  \leq \sqrt{\log(\abs{\cE} \abs{\cX} / \delta) / (2 n_{a, X})}
  \label{eq:joint concentration}
\end{align}
holds with probability at least $1 - \delta$, jointly over all items $a$ and contexts $X$.

Now, we prove our main claim. For any context $X$ and list $A = (a_k)_{k \in [K]}$, we have from the definition of the PBM that
\begin{align*}
  V_\textsc{pbm}(A, X) - \hat{V}_\textsc{pbm}(A, X)
  = \sum_{k = 1}^K p_{k,X} (\theta_{a_k, X} - \hat{\theta}_{a_k, X})\,.
\end{align*}
Since $p_{k,X} \in [0, 1]$, we have
\begin{align}
  \abs{V_\textsc{pbm}(A, X) - \hat{V}_\textsc{pbm}(A, X)}
  \leq \sum_{k = 1}^K |\theta_{a_k, X} - \hat{\theta}_{a_k, X}|
  = \sum_{a \in A} |\theta_{a, X} - \hat{\theta}_{a, X}|
  \leq \sum_{a \in A} \sqrt{\log(\abs{\cE} \abs{\cX} / \delta) / (2 n_{a, X})}\,.
\end{align}
The last inequality is under the assumption that \eqref{eq:joint concentration} holds, which holds with probability at least $1 - \delta$.

\paragraph{CM}
To prove the bound for CM and DCM, we first show how the difference of two products with $K$ variables is bounded by the difference of their sums.

\begin{lemma}
\label{lemma: products bounded by sum}
Let $(a_k)_{k=1}^K \in [0, 1]^K$ and $(b_k)_{k=1}^K \in [0, 1]^K$. Then
\begin{align*}
    \abs{\prod_{k=1}^Ka_k - \prod_{k=1}^Kb_k}
    \leq \sum_{k=1}^K\abs{a_k - b_k}.
\end{align*}
\end{lemma}

\begin{proof}
We start with
\begin{align*}
  \prod_{k = 1}^K a_k - \prod_{k = 1}^K b_k
  & = \prod_{k = 1}^K a_k - a_1 \prod_{k = 2}^K b_k +
  a_1 \prod_{k = 2}^K b_k - \prod_{k = 1}^K b_k
  = a_1 \left(\prod_{k = 2}^K a_k - \prod_{k = 2}^K b_k\right) +
  (a_1 - b_1) \prod_{k = 2}^K b_k \\
  & \stackrel{(a)}{=} \sum_{k = 1}^K \left(\prod_{i = 1}^{k - 1} a_i\right)
  (a_k - b_k) \left(\prod_{i = k + 1}^K b_i\right)\,,
\end{align*}
where $(a)$ is by recursively applying the same argument to $\prod_{k = 2}^K a_k - \prod_{k = 2}^K b_k$. Now, we apply the absolute value and get
\begin{align*}
  \abs{\prod_{k = 1}^K a_k - \prod_{k = 1}^K b_k}
  = \abs{\sum_{k = 1}^K \left(\prod_{i = 1}^{k - 1} a_i\right)
  (a_k - b_k) \left(\prod_{i = k + 1}^K b_i\right)}
  \leq \sum_{k = 1}^K \abs{a_k - b_k}\,,
\end{align*}
since $\prod_{i = 1}^{k - 1} a_i \in [0, 1]$ and $\prod_{i = k + 1}^K b_i \in [0, 1]$.

\end{proof}

Now, we prove our main claim. For any context $X$ and list $A = (a_k)_{k \in [K]}$, we have from the definition of the CM and from the bound in \cref{lemma: products bounded by sum} that 
\begin{align*}
    \abs{V_\textsc{cm}(A, X) - \hat{V}_\textsc{cm}(A, X)} 
    &= \abs{1 - \prod_{k = 1}^K (1 - \theta_{a_k, X}) - 1 + \prod_{k = 1}^K (1 - \hat{\theta}_{a_k, X})} 
    = \abs{\prod_{k = 1}^K (1 - \hat{\theta}_{a_k, X}) - \prod_{k = 1}^K (1 - \theta_{a_k, X})} \\
    &\leq \sum_{k = 1}^K |\theta_{a_k, X} - \hat{\theta}_{a_k, X}| = \sum_{a \in A} |\theta_{a, X} - \hat{\theta}_{a, X}|
    \leq \sum_{a \in A} \sqrt{\log(\abs{\cE} \abs{\cX} / \delta) / (2 n_{a, X})}\,.
\end{align*}
The last inequality is under the assumption that \eqref{eq:joint concentration} holds, which holds with probability at least $1 - \delta$.

\paragraph{DCM}
For any context $X$ and list $A = (a_k)_{k \in [K]}$, we have from the definition of the DCM and from the bound in \cref{lemma: products bounded by sum} that 
\begin{align*}
    \abs{V_\textsc{dcm}(A, X) - \hat{V}_\textsc{dcm}(A, X)} 
    &= \abs{1 - \prod_{k = 1}^K (1 -  (1- \lambda_{k,X})\theta_{a_k, X}) - 1 + \prod_{k = 1}^K (1 - (1- \lambda_{k,X})\hat{\theta}_{a_k, X})}\\ 
    &= \abs{\prod_{k = 1}^K (1 -  (1- \lambda_{k,X})\hat{\theta}_{a_k, X}) - \prod_{k = 1}^K (1 -  (1- \lambda_{k,X})\theta_{a_k, X})} \\
    &\leq \sum_{k = 1}^K |\theta_{a_k, X} - \hat{\theta}_{a_k, X}| = \sum_{a \in A} |\theta_{a, X} - \hat{\theta}_{a, X}|
    \leq \sum_{a \in A} \sqrt{\log(\abs{\cE} \abs{\cX} / \delta) / (2 n_{a, X})}\,.
\end{align*}
The first inequality holds as $\lambda_{k,X} \in [0, 1]$, and the last inequality is under the assumption that \eqref{eq:joint concentration} holds, which holds with probability at least $1 - \delta$. This completes the proof.

\paragraph{Proof of \cref{lemma:optimization lcb}}
\begin{proof}
Let $L(A, X) = \hat{V}(A, X) - c(A, X)$. Suppose that $\abs{\hat{V}(A, X) - V(A, X)} \leq c(A, X)$ holds for all $A$ and $X$. Note that this also implies $L(A, X) \leq V(A, X)$. Then, for any $X$,
\begin{align*}
  V(A_{*, X}, X) - V(\hat{A}_X, X)
  & = V(A_{*, X}, X) - L(A_{*, X}, X) + L(A_{*, X}, X) - V(\hat{A}_X, X) \\
  & \leq V(A_{*, X}, X) - L(A_{*, X}, X) + L(\hat{A}_X, X) - V(\hat{A}_X, X) \\
  & \leq V(A_{*, X}, X) - L(A_{*, X}, X) \\
  & = V(A_{*, X}, X) - \hat{V}(A_{*, X}, X) + c(A_{*, X}, X) \\
  & \leq 2 c(A_{*, X}, X)\,.
\end{align*}
The first inequality holds because $\hat{A}_X$ maximizes $L(\cdot, X)$. The second inequality holds because $L(\hat{A}_X, X) \leq V(\hat{A}_X, X)$. The last inequality holds because $V(A_{*, X}, X) - \hat{V}(A_{*, X}, X) \leq c(A_{*, X}, X)$.

It remains to prove that $\abs{\hat{V}(A, X) - V(A, X)} \leq c(A, X)$ holds for all $A$ and $X$. By \cref{lemma: concentration list}, this occurs with probability at least $1 - \delta$, jointly over all $A$ and $X$, for $c(A, X)$ in \cref{lemma: concentration list}. This completes the proof.
\end{proof}

\section{Additional Experiments}
\label{appendix: other experiments}

\paragraph{Most Frequent Query}
In \cref{fig:01}, we show the simplest case when using only the most frequent query to observe the effect of LCBs without possible skewing due to averaging over multiple queries. Results found in the most frequent query support those already discussed in \emph{Top 10 Queries}.

\paragraph{Less Frequent Queries}
We study how less frequent queries impact the performance of LCBs. We use the same setup as in the \emph{Top 10 Queries} experiment and observe how the error changes as the number of queries increases. In \cref{fig:03}, we show comparable improvements to those of DCM in \cref{fig:02}, showing that LCBs perform well even in less frequent queries. We performed this experiment with CM and PBM as well and observed a similar behavior.

\begin{figure*}[h]
    \centering
    \begin{subfigure}[t]{0.8\linewidth}
     \includegraphics[width=\textwidth]{figs/legend.pdf}     
    \end{subfigure}
    \begin{subfigure}[t]{\linewidth}
         \includegraphics[width=\textwidth]{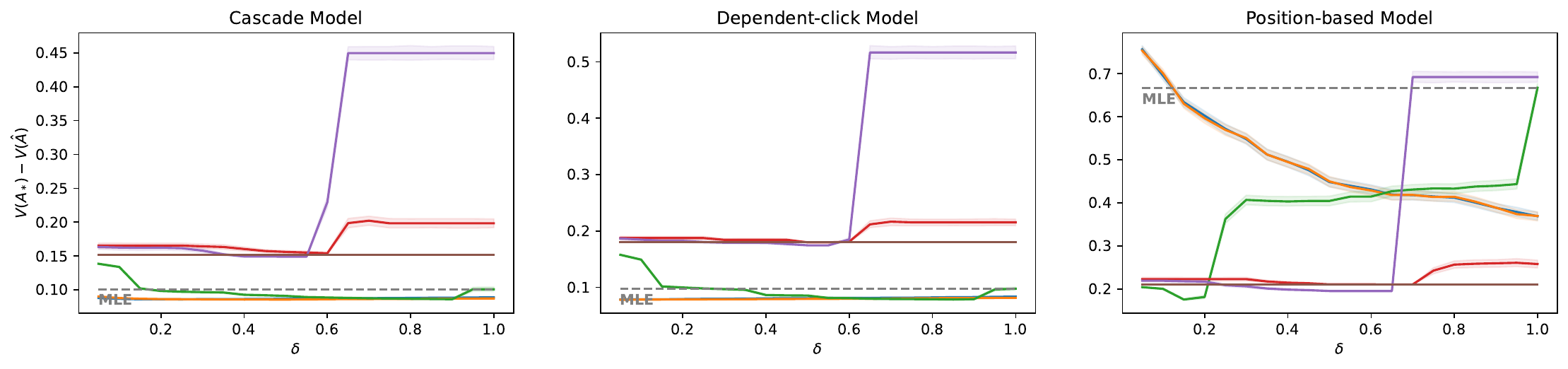}
    \end{subfigure}
    \caption{Error of the estimated lists $\hat{A}$ compared to optimal lists $A_*$ on the most frequent query.}
    \label{fig:01}
\end{figure*}

\begin{figure*}[h]
    \centering
    \begin{subfigure}[t]{0.8\linewidth}
     \includegraphics[width=\textwidth]{figs/legend.pdf}     
    \end{subfigure}
    \begin{subfigure}[t]{\linewidth}
        \includegraphics[width=\textwidth]{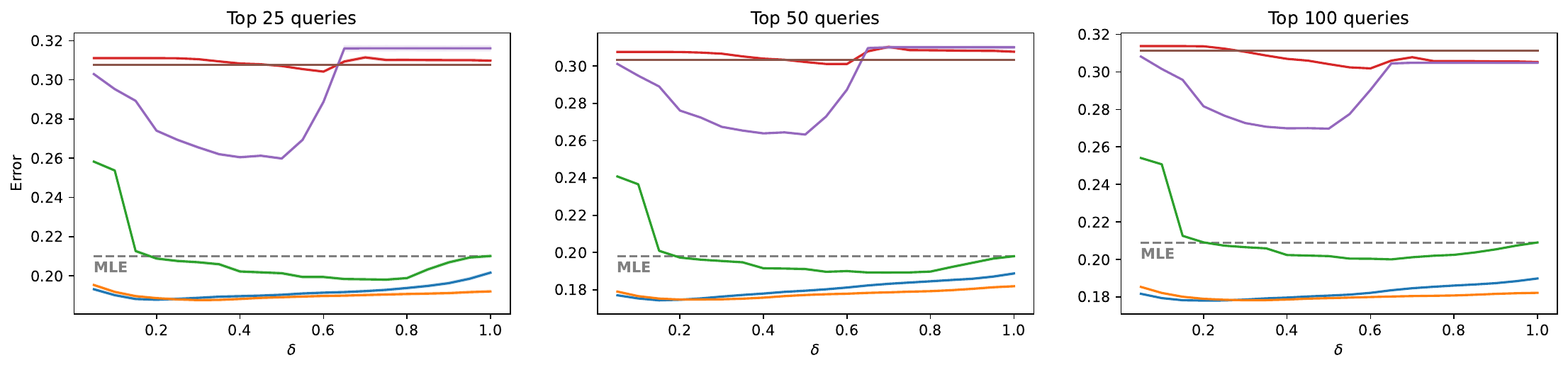}
    \end{subfigure}
    \caption{Comparison of our methods to baselines on DCM and less frequent queries.}
    \label{fig:03}
\end{figure*}

\paragraph{Cross-Validation Setting}
% \label{sec:cross-validation setting}
In all of the previous experiments, we fit a model on logged data $\cD$, use that model to generate new clicks, and then learn a new model using those clicks. One can argue that data generated under the right model yields better results, and under the wrong model, LCBs do not hold. Therefore, we split logged data $\cD$ on the 23rd day to train and test the set and use the first 23 days to train model parameters. We randomly select 1000 samples for each query from the remaining 4 days to fit models using LCB while varying $\delta$. The rest of the setting is the same as in the \emph{Top 10 Queries} experiment. The reason for this split is that using the remaining four days provides enough data in the top 10 queries to sample from. We tested other split ratios and observed that the results did not differ significantly. In \cref{fig:05}, we observe the Bayesian LCBs outperform all other baselines, and reasoning under uncertainty is still better than using MLE. Similarly to the previous experiments, we omit other linear baselines in CM and DCM plots as they perform considerably worse.

\begin{figure*}[h]
    \centering
    \begin{subfigure}[t]{0.8\linewidth}
     \includegraphics[width=\textwidth]{figs/legend.pdf}     
    \end{subfigure}
    \begin{subfigure}[t]{\linewidth}
        \includegraphics[width=\textwidth]{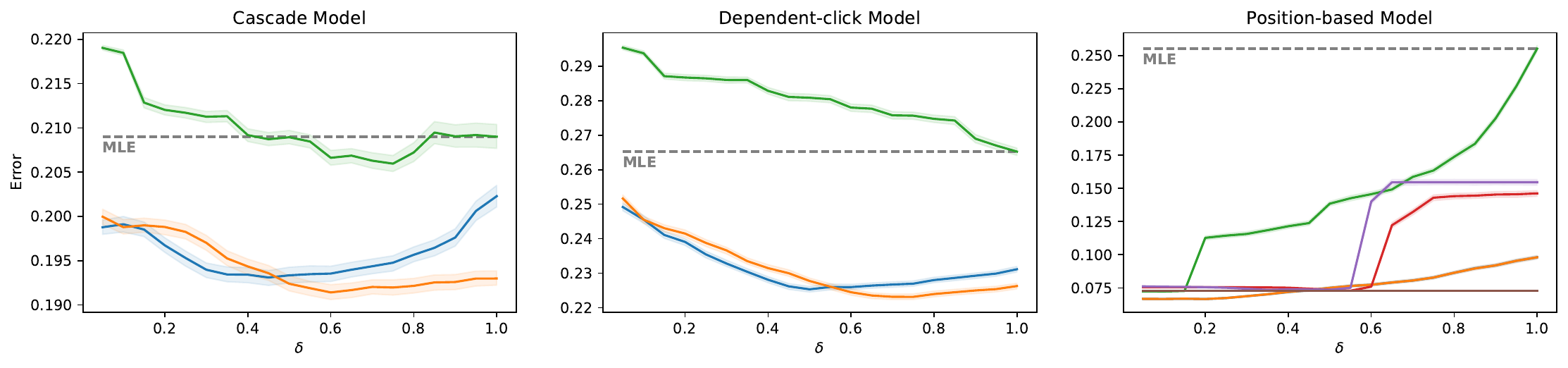}
    \end{subfigure}
    \caption{Comparison of our methods against other baselines in a setting where the fitted ground truth model does not generate clicks.} 
    \label{fig:05}
\end{figure*}

\paragraph{Finding Prior}
When we estimated prior in the previous experiments, we did so on a grid $(\alpha, \beta) \in \cG^2, \; \cG = \{2^i\}_{i=1}^{10}$ and we observed significant improvements over fixed prior $\mathrm{Beta}(1, 1)$ mostly when $\delta = 1$. In this experiment, the goal is to measure an improvement with the increasing grid size. We evaluate DCM using Bayesian LCBs and keep the rest of the setting the same as in the \emph{Top 10 Queries} experiment. We then estimate prior from these queries by searching over grid $ \cG^2, \; \cG = \{2^{i-1}\}_{i=1}^m, \; m \in [1, 2, 5, 10, 20]$. In \cref{fig:06}, we see how the results get more robust against the $\delta$ hyperparameter until $\abs{\cG} = 10$, and after that, the method still finds the same optimal prior. Therefore, in our case, it is sufficient to search over $\cG = \{2^i\}_{i=1}^{10}$. The optimal values found on $\cG$ are $ \theta \sim \mathrm{Beta}(1, 8)$ and $\lambda \sim \mathrm{Beta}(1, 64)$.

\begin{figure*}[h]
    \centering
    \includegraphics[width=0.5\textwidth]{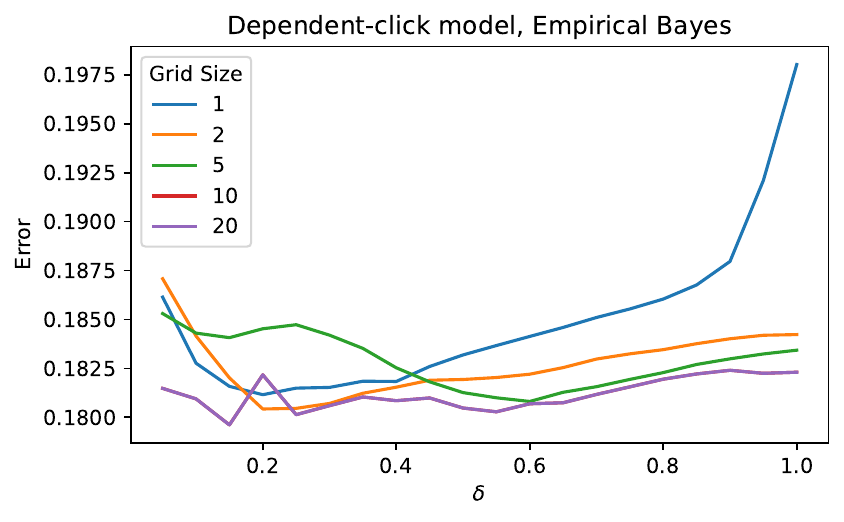}
    \caption{Empirical Bayes showing increasing performance with larger grid size. Red and purple lines are overlapping.} 
    \label{fig:06}
\end{figure*}

\end{document}